\documentclass{article}

\PassOptionsToPackage{numbers,sort&compress}{natbib}
\usepackage[final,sglblindworkshop]{neurips_2026}
\workshoptitle{E-Values: From Statistics to ML}
\usepackage{amsmath,amssymb,amsthm,booktabs,microtype,hyperref,tikz}
\usetikzlibrary{arrows.meta,positioning,calc}

\newtheorem{theorem}{Theorem}
\newtheorem{proposition}{Proposition}
\newtheorem{lemma}{Lemma}

\theoremstyle{definition}
\newtheorem{definition}{Definition}

\theoremstyle{remark}

\newcommand{\E}{\mathbb E}
\newcommand{\F}{\mathcal F}
\newcommand{\R}{\mathbb R}

\newcommand{\fstar}{f^\star}

\title{Prequential E-Values for\\Selected-GP Near-Optimality Certificates}
\author{
  Ami Tavory\\
  Meta Platforms
  \And
  Noa Cohen\\
  Meta Platforms
}
\date{}

\begin{document}
\maketitle

\begin{abstract}
When optimizing an expensive black-box function sequentially, as in
hyperparameter optimization, we may want to stop once the best evaluated value
is certified within $\varepsilon$ of the global optimum. Such a certificate
needs two ingredients: a lower confidence bound for the selected value and an
upper confidence envelope over the domain, typically supplied by a
Gaussian process (GP). GP-UCB-style stopping rules are valid when the kernel and
constants defining this envelope are fixed before the run, but the practical
temptation is to tune the envelope from the same adaptive evaluations and then
certify as if it had been fixed. We use prequential e-values to make this selection
auditable: starting from a predeclared set of fully specified GP/RKHS
envelopes, each candidate is tested by its own one-step-ahead e-process,
contradicted candidates are deleted, and certification uses the largest upper
bound among the survivors. 
With a valid selected-point lower bound and one declared candidate having valid
latent coverage and noise calibration, the rule is anytime-valid.
On a 512-seed noisy RBF stress sweep, it roughly halves false-certification
risk at comparable power versus fit-then-certify. Relative to random fixed GP
precommitment on smooth $d=3,4$ objectives, each additional false certificate
is accompanied by 3.0 and 13.5 additional correct certificates, respectively.
\end{abstract}

\section{Optimization and GP Certificates}

Many ML workflows, including hyperparameter optimization and AutoML, use a
sequential optimizer \citep{snoek2012practicalbo,akiba2019optuna,botorch2020} to
query a noisy function at configurations
$\lambda_1,\lambda_2,\ldots\in\Lambda$. At time $t$, the optimizer returns the
current best queried configuration $\hat\lambda_t$, whose observed score
attains $\max_{t'\le t}Y_{t'}$, where $Y_{t'}$ is a noisy measurement of
$f(\lambda_{t'})$. We are often interested in whether the best configuration
found so far is within $\varepsilon$ of the global optimum:
\begin{equation}
\label{eq:near-optimality-target}
  f(\hat\lambda_t) \ge f^\star-\varepsilon,\qquad
  f^\star=\sup_{\lambda\in\Lambda} f(\lambda).
\end{equation}
To that end, we can certify
\begin{equation}
\label{eq:stop-certificate}
  L_t(\hat\lambda_t) \ge U_t-\varepsilon,
\end{equation}
where
$L_t(\hat\lambda_t)$ is a lower bound on the selected configuration's value, and
$U_t$ is an upper bound on the unobserved global optimum.

Under an observation-model assumption, the lower bound is a standard winner's-curse/post-selection problem caused by choosing the best observed value
\citep{cawleyTalbot2010overfitting,zrnicFithian2025winnerscurse,zhangLeeLei2024winners}
(see Appendix~\ref{app:selected-lower}).
The upper bound requires a further structure assumption: without one, an
arbitrary optimum can occur in any unqueried configuration.
At a fixed prefix $t$, GP-UCB and related Bayesian-optimization stopping rules
use a fixed kernel/RKHS model to build
$U_t=\sup_{\lambda\in\Lambda}[\mu_t(\lambda)+\beta_t\sigma_t(\lambda)]$,
an upper bound from the posterior mean and standard deviation
\citep{rasmussenWilliams2006gp,srinivas2010gpucb,chowdhuryGopalan2017kernelized,wilson2024stoppingBO,wangWangWei2026stoppingBO,berkenkamp2019unknownHyperparams}.
These give valid near-optimality rules through
\eqref{eq:stop-certificate} under a fixed GP assumption. Fixing the GP can make
the certificate brittle, but selecting or tuning its kernel, radius, or
inflation from the same data and then applying the fixed-GP guarantee can
increase false-certification risk.

We address the upper bound by treating GP-envelope selection as a post-selection
certification problem rather than ordinary model fitting, building on the
post-selection advantages of e-values
\citep{cawleyTalbot2010overfitting,xuWangRamdas2024posi}. Before seeing the run, we declare $m$
fully specified GP/RKHS envelope candidates. As the optimizer reveals ordered
observations, each candidate band is challenged by one-step-ahead residuals: it
must bound the next queried configuration before its value is observed. The
residuals are accumulated into one e-process per candidate
\citep{vovkWang2021evalues,ramdas2023gavi}; when a candidate's e-process
crosses $1/\delta_U$, where $\delta_U$ is the predeclared upper-side audit
error budget, that candidate is deleted.

The certificate combines the largest upper bound among the surviving bands
with a selected-point lower confidence bound for $\hat\lambda_t$ (bounded-score
KL in our finite-validation experiments). Using the largest survivor preserves
validity without selecting a favorably tight band. For any fixed declared candidate whose envelope
has all-time upper coverage error at most $\alpha_U$ and whose one-step audit
null is correct, Ville's inequality \citep{ville1939} bounds the probability of
ever deleting that candidate by $\delta_U$. Since the certificate takes the largest survivor,
one such surviving anchor is enough. If the selected-configuration lower bound
fails with probability at most $\alpha_L$, then any certificate issued at any
stopping time $\tau$ satisfies
\[
  f(\hat\lambda_\tau)\ge f^\star-\varepsilon
\]
with probability at least $1 - (\alpha_L+\alpha_U+\delta_U)$.

\paragraph{Contributions.}
We contribute: (1) a
post-selection certificate for global near-optimality under a predeclared set of
GP/RKHS envelope assumptions; (2) a prequential e-value screening rule that
makes data-dependent envelope selection auditable while preserving anytime
validity; and (3) known-optimum experiments showing that this can beat a random
fixed precommitment from the same declared set in smooth, certifiable regimes,
while tightest-survivor and plug-in variants leak risk.

\begin{figure}[t]
\centering
\begin{tikzpicture}[x=1cm,y=1cm]
  \input{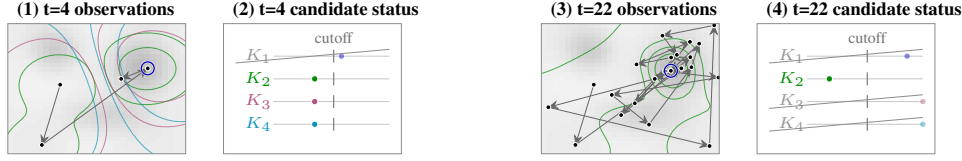}
\end{tikzpicture}
\caption{Grey surface: hidden objective. Colored contours: surviving
candidate GP fits. Black path: optimizer observations. Candidate-status
panels: prequential evidence dots, deletion cutoff, and grey deletion slashes.}
\label{fig:system-diagram}
\end{figure}

\section{GP Upper Bounds and Prequential E-Value Audits}
\label{sec:method}

We construct the upper bound of \eqref{eq:stop-certificate}. Before the optimization,
we declare a set of $m$ fully specified GP/RKHS envelope candidates
$K_1,\ldots,K_m$. Candidate $K_j$ fixes its kernel, kernel parameters, noise or
RKHS-radius constants, and confidence inflation at upper-side level
$\alpha_U$. After the first $t$ evaluations, candidate $K_j$ is updated only
with the prefix observed so far and gives
\[
  u^f_{j,t}(\lambda)=\mu_{j,t}(\lambda)+\beta_{j,t}\sigma_{j,t}(\lambda),
  \qquad
  U_{j,t}=\sup_{\lambda\in\Lambda} u^f_{j,t}(\lambda).
\]
If $K_j$ is a valid envelope assumption for the objective, then $U_{j,t}$ is
the fixed-GP upper bound that can certify the unqueried domain.

Figure~\ref{fig:system-diagram} illustrates the procedure. The true structure
(shown in grey) is unknown. At step $t$ a number $\le m$ of GPs survive, and
each updates its posterior envelope from the observed prefix, with its audit
based on one-step-ahead prediction (colored contours in panels (1) and (3)). Each live
candidate compares the next observation against its envelope, and is eliminated
if that observation exceeds its cutoff (panels (2) and (4)).
Deletion acts as a validity screen, after which the certificate uses the
largest upper bound among the surviving candidates (panel (4)).

\paragraph{Single-candidate e-value.}
Let $\mathcal F_t$ denote the information available after the first $t$
evaluations. Fix one candidate $K_j$. At step $t$, before $Y_t$ is observed,
candidate $j$ forms its latent band using only $\mathcal F_{t-1}$ and uses its
observation-noise model to add a one-step buffer $r_j(q_j)$ satisfying
\[
\begin{aligned}
  \Pr\bigl(
    Y_t>f(\lambda_t)+r_j(q_j)
    \mid \mathcal F_{t-1}
  \bigr)
  &\le q_j .
\end{aligned}
\]
After $Y_t$ is revealed, the violation indicator
\[
\begin{aligned}
  V_{j,t}
  &=
  \mathbf 1\!\bigl(
    Y_t>u^f_{j,t-1}(\lambda_t)+r_j(q_j)
  \bigr)
\end{aligned}
\]
challenges that candidate. Under the on-path null for this candidate,
$\Pr(V_{j,t}=1\mid\mathcal F_{t-1})\le q_j$. A fixed bet
$\eta_j\in[0,1/q_j)$ then makes
\[
  M_{j,t}=1+\eta_j(V_{j,t}-q_j)
\]
a nonnegative one-step e-value increment, since
\[
  \E(M_{j,t}\mid\mathcal F_{t-1})
  =
  1+\eta_j\bigl(\E(V_{j,t}\mid\mathcal F_{t-1})-q_j\bigr)
  \le 1 .
\]
Multiplying these increments gives the candidate's e-process
\[
  E_{j,t}=E_{j,t-1}\bigl(1+\eta_j(V_{j,t}-q_j)\bigr),
  \qquad E_{j,0}=1.
\]
Large values are evidence against the candidate's one-step calibration.
Deletion at $E_{j,t}\ge 1/\delta_U$ is therefore an e-value test, and Ville's
inequality gives deletion probability at most $\delta_U$ for any candidate
satisfying the on-path null, uniformly over stopping times.

\paragraph{Combination.}
We run the one-candidate audit separately for each declared $K_j$ and let
$A_t$ be the candidates not deleted by their e-processes. The certificate uses
the largest surviving upper bound,
\begin{equation}
\label{eq:survivor-envelope}
  U_t^{\rm surv}=\max_{j\in A_t} U_{j,t},
  \qquad \max\varnothing:=+\infty.
\end{equation}
Thus the procedure abstains if every candidate is deleted.
This conservative choice avoids selecting a lucky under-covering survivor;
choosing the tightest fitted GP after seeing the same data would condition on
success and then reuse the fixed-GP guarantee.
Fix any predeclared anchor $j^\star$ whose latent band covers $f$ uniformly
over the domain and time with failure probability at most $\alpha_U$. On this
event, $U_{j^\star,t}\ge f^\star$ and
\[
  V_{j^\star,t}
  \le
  \mathbf 1\!\bigl(
    Y_t>f(\lambda_t)+r_{j^\star}(q_{j^\star})
  \bigr).
\]
Replacing $u^f_{j^\star,t-1}(\lambda_t)$ by $f(\lambda_t)$ in every violation
factor defines an unconditional comparison e-process that, on the coverage
event, pathwise dominates $E_{j^\star,t}$. Ville therefore bounds deletion of the anchor before any
coverage failure by $\delta_U$. With probability at least
$1-\alpha_U-\delta_U$, the anchor both
covers the optimum and is never deleted, so
$U_t^{\rm surv}\ge U_{j^\star,t}\ge f^\star$ at every step. We stop only when
\eqref{eq:stop-certificate} holds with $U_t=U_t^{\rm surv}$. Combining this
upper side with the lower-side error $\alpha_L$ gives
\begin{equation}
\label{eq:alpha-allocation}
  \Pr(\exists t':\hbox{ false }\varepsilon\hbox{-certificate at }t')
  \le \alpha_L+\alpha_U+\delta_U .
\end{equation}
The number of candidates $m$ affects computation and survivor looseness, not
the risk budget.

\section{Experiments}

We use known-optimum analytic objectives so each stop can be scored by the
power/risk pair $a/b$, where $a=\Pr(\mathrm{certify\ and\ correct})$ and
$b=\Pr(\mathrm{certify\ and\ wrong})$. All methods use the same optimizer
runs and selected-configuration lower bound; only the upper certificate
changes. The experiment includes
finite-validation noise and GP-kernel objectives in low dimension. The
experiment protocol is in Appendix~\ref{app:canonical}, and a reference
implementation is available on GitHub\footnote{Reference implementation:
\url{https://github.com/atavory/gp_evalues}.}. For interpretation,
Figure~\ref{fig:fixed-fit-qcut-stack} bins rows by
$\rho=64^{-1/d}/\ell_{\rm rms}$, an analysis-only proxy for how coarse the
early optimizer run is relative to the function scale.

\paragraph{Random fixed precommitment.}
The fixed baseline samples one fully specified envelope from the same
predeclared set before the run and never adapts it. We compare the change in
correct-certification probability to the change in false-certification
probability. Aggregated over the $d=3$ and $d=4$ fixed-precommitment experiments,
the audited certificate gives
3.0 and 13.5 additional correct certificates per additional false certificate
in $d=3$ and $d=4$, respectively. The $\rho$ breakdown shows where the gain comes from. In the
low and middle tertiles, the audit turns many fixed abstentions into
certificates; in the hardest tertile, both methods mostly abstain.

\begin{center}
\refstepcounter{figure}\label{fig:fixed-fit-qcut-stack}
\includegraphics[width=0.98\linewidth]{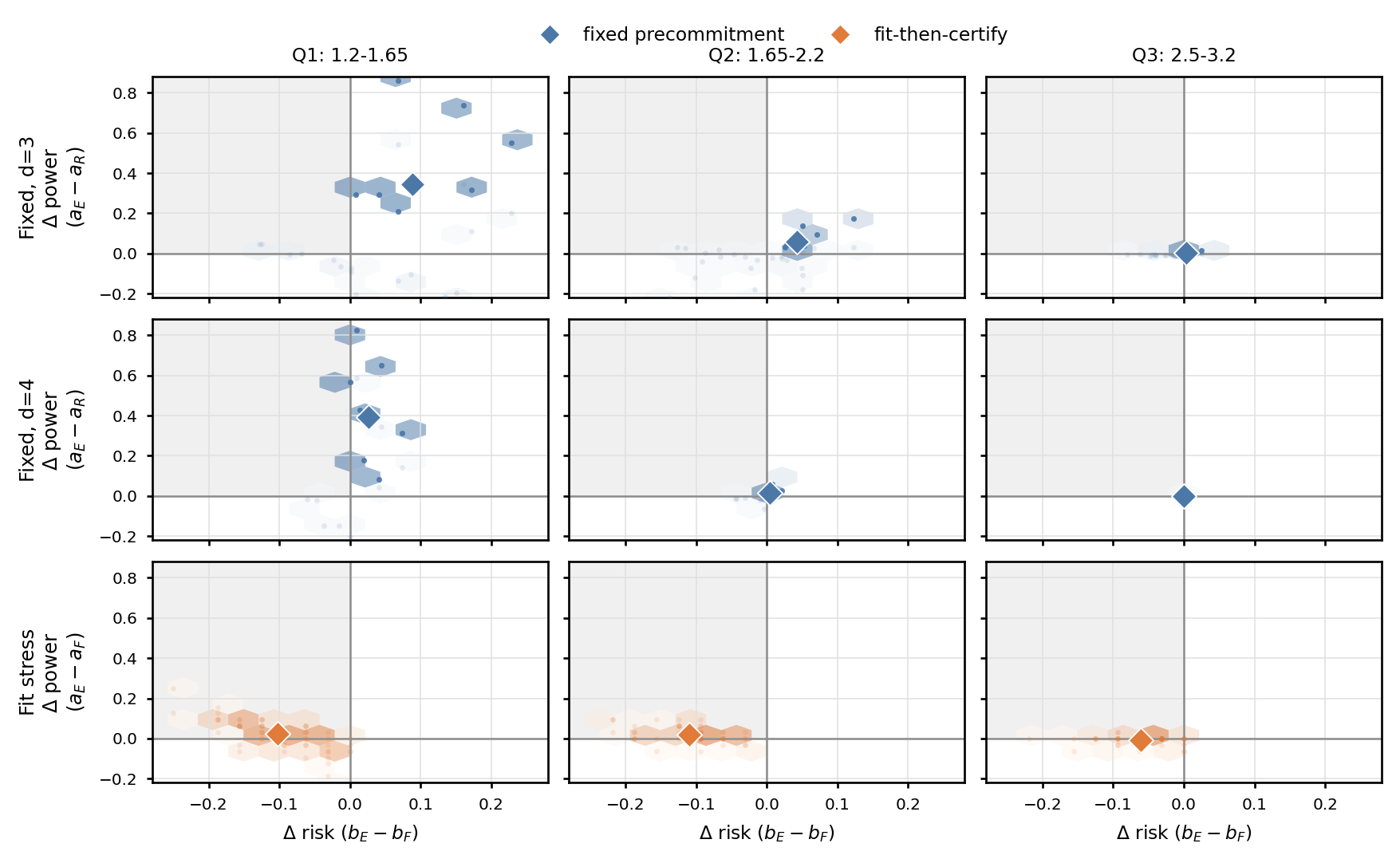}
\vspace{-0.35em}
{\small\textbf{Figure~\thefigure.}
Paired advantage of the e-value audited certificate (E) across row-wise $\rho$
tertiles. Upper block: comparison with random fixed precommitment (R), for
$d=3,4$. Lower block: comparison with fit-then-certify (F) on the 512-seed RBF
stress sweep. Blue marks fixed comparisons and orange marks fit-then-certify;
top-left favors E.}
\end{center}

\paragraph{Fit-then-certify.}
The fit-then-certify diagnostic fits or selects the GP envelope from the same
history and then certifies as if it had been fixed in advance. Here the natural
comparison is false-certification risk at matched certification power. On the
512-seed RBF length-scale stress sweep, the audited and fit-then-certify
power/risk pairs are $0.280/0.091$ and $0.267/0.183$. Their risk-per-correct
ratios are $0.325$ and $0.685$, respectively, giving a ratio of $0.47$.
Across $\rho$ tertiles, the
risk reduction persists in every bin; the power tradeoff appears mainly in the
hardest tertile. Appendix~\ref{app:fit-then-certify-risk} explains why this
risk is upper-side undercoverage.

\section{Discussion and limitations}

Prequential e-values let near-optimality certificates use a data-selected
subset of a predeclared GP-envelope family while preserving validity. The
search policy is unchanged; the contribution is post-selection accounting for
the certificate, under the condition that the declared envelope family contains
a valid member. We handle a finite family; continuous hyperparameter families
would require additional uniform or mixture accounting.

The main limitation is geometric; Appendix~\ref{app:dimension-resolution},
including Theorem~\ref{thm:app-imposs}, formalizes this barrier. Both the experiments and GP/RKHS theory
reflect the $n^{-1/d}$ fill-distance barrier: with tens or hundreds of
evaluations, global upper envelopes become weak beyond a few active
hyperparameters. In that regime the valid outcome is often abstention unless
one declares lower-dimensional structure or accepts model-trusting
extrapolation.

\clearpage
\bibliographystyle{plainnat}
\bibliography{references}

\clearpage
\appendix
\section{Selected-Configuration Lower Bounds}
\label{app:selected-lower}

Choosing the returned configuration from noisy observed scores is a classic
winner's curse problem for noisy comparison
\citep{cawleyTalbot2010overfitting,zrnicFithian2025winnerscurse,zhangLeeLei2024winners}.
The certificate only needs a one-sided lower bound for each queried
configuration: for any fixed $\lambda_{t'}$ and level $\alpha$, a statistic
$L_{t'}(\alpha)$ with
\[
  \Pr(L_{t'}(\alpha)>f(\lambda_{t'}))\le\alpha .
\]
Selection then costs only a union bound over configurations that could be
returned.

\paragraph{Bounded-score lower bound.}
Suppose each queried configuration has one observed mean score
\[
  Y_{t'}=\frac{1}{n_{\rm val}}\sum_{s=1}^{n_{\rm val}}Z_{t',s},
  \qquad Z_{t',s}\in[0,1],\quad \E Z_{t',s}=f(\lambda_{t'}),
\]
with conditionally independent validation draws. Bernoulli accuracy scores are
the special case used in our finite-validation experiments.
For $y,q\in[0,1]$, write
\[
  {\rm kl}(y,q)=y\log\frac{y}{q}+(1-y)\log\frac{1-y}{1-q}.
\]

\begin{proposition}[Bounded-score KL lower bound]
\label{prop:app-kl}
For a fixed configuration with mean $p$, define
\[
  L_{\rm KL}(Y,n_{\rm val},\alpha)
  =
  \inf\{q\le Y:n_{\rm val}{\rm kl}(Y,q)\le \log(1/\alpha)\}.
\]
Then $\Pr(L_{\rm KL}(Y,n_{\rm val},\alpha)>p)\le\alpha$.
\end{proposition}

\begin{proof}
The failure event can occur only on the upper tail of the bounded mean:
\[
  \{L_{\rm KL}(Y,n_{\rm val},\alpha)>p\}
  \subseteq
  \{Y>p,\ n_{\rm val}{\rm kl}(Y,p)>\log(1/\alpha)\}.
\]
For any $y\ge p$, Hoeffding's bounded-variable Chernoff bound gives
\[
  \Pr(Y\ge y)
  \le
  \exp(-n_{\rm val}{\rm kl}(y,p)).
\]
Let
\[
  y_\alpha=\inf\{y\ge p:n_{\rm val}{\rm kl}(y,p)\ge\log(1/\alpha)\}.
\]
Since $y\mapsto{\rm kl}(y,p)$ is increasing on $[p,1]$, the failure event is
contained in $\{Y\ge y_\alpha\}$, and therefore
\[
  \Pr(L_{\rm KL}(Y,n_{\rm val},\alpha)>p)
  \le
  \Pr(Y\ge y_\alpha)
  \le
  \alpha .
\]
\end{proof}

\begin{proposition}[Anytime selection accounting]
\label{prop:app-selected}
Fix weights $w_{t'}\ge0$ with
$\sum_{t'\ge1}w_{t'}\le1$. For each queried configuration, let
$L_{t'}(\alpha_Lw_{t'})$ be any lower bound satisfying
$\Pr(L_{t'}(\alpha_Lw_{t'})>f(\lambda_{t'}))\le\alpha_Lw_{t'}$; in the
bounded-score example, take
$L_{t'}(\alpha_Lw_{t'})=L_{\rm KL}(Y_{t'},n_{\rm val},\alpha_Lw_{t'})$. Then
\[
  \Pr(\exists t'\ge1:L_{t'}(\alpha_Lw_{t'})>f(\lambda_{t'}))\le\alpha_L .
\]
\end{proposition}
Consequently, at any finite stopping time $t$, any data-dependent returned
configuration $\hat\lambda_t\in\{\lambda_1,\ldots,\lambda_t\}$ is covered on
the simultaneous event. A finite predeclared horizon is recovered by assigning
zero weight after that horizon.

The fixed-configuration guarantee requires an observation assumption for a
nontrivial finite-sample lower confidence bound. The bounded-score KL bound is
the instantiation used in our experiments; other metrics require their
corresponding one-sided lower bounds.

\section{GP Upper Bounds and Prequential Audits}
\label{app:accounting}

Let $\Lambda\subset\R^d$ be compact, let $f:\Lambda\to\R$ be the latent
objective, and let $\fstar=\sup_{\lambda\in\Lambda}f(\lambda)$. A sequential
policy chooses configurations $\lambda_t$ and observes noisy scores $Y_t$.
Let $\F_t$ be the information available after the first $t$ evaluations.

\begin{definition}[Envelope candidate]
\label{def:app-envelope-candidate}
A candidate $K_j$ is fixed before the run. It specifies a kernel and its
parameters, an RKHS-radius or prior constant, an observation-noise tail bound,
and an inflation schedule. After the first $t$ observations it gives a
predictable latent envelope
\[
  u^f_{j,t}(\lambda)=\mu_{j,t}(\lambda)+\beta_{j,t}\sigma_{j,t}(\lambda),
  \qquad
  U_{j,t}=\sup_{\lambda\in\Lambda}u^f_{j,t}(\lambda).
\]
For its one-step audit at step $t$, it also gives the noisy-score threshold
\[
  c^Y_{j,t}
  =
  u^f_{j,t-1}(\lambda_t)+r_j(q_j),
\]
where $r_j(q_j)$ is the declared noise quantile. The usual GP-UCB/RKHS choices
of $\beta_{j,t}$ give all-time latent coverage under their fixed assumptions
\citep{srinivas2010gpucb,chowdhuryGopalan2017kernelized}.
\end{definition}

\paragraph{Single-candidate e-value.}
For one fixed candidate, the audit only uses one-step claims made before the
next response is observed. The following two lemmas are the e-value part of the
construction.

\begin{lemma}[Pathwise audit domination]
\label{lem:app-coverage-audit-null}
Suppose candidate $j$ has latent upper coverage
\[
  \mathcal C_j
  :=
  \{\forall t,\forall\lambda\in\Lambda:
  u^f_{j,t}(\lambda)\ge f(\lambda)\},
  \qquad
  \Pr(\mathcal C_j^c)\le\alpha_U,
\]
and its noise quantile satisfies
\[
  \Pr\bigl(Y_t>f(\lambda_t)+r_j(q_j)\mid\F_{t-1}\bigr)
  \le q_j .
\]
Define the observable violation and an unobservable comparison violation by
\[
  V_{j,t}
  =
  \mathbf 1\!\bigl(Y_t>c^Y_{j,t}\bigr),
  \qquad
  \widetilde V_{j,t}
  =
  \mathbf 1\!\bigl(Y_t>f(\lambda_t)+r_j(q_j)\bigr).
\]
Then $\Pr(\widetilde V_{j,t}=1\mid\F_{t-1})\le q_j$, and on
$\mathcal C_j$, $V_{j,t}\le\widetilde V_{j,t}$ for every $t$.
\end{lemma}

\begin{proof}
\[
  \Pr(\widetilde V_{j,t}=1\mid\F_{t-1})\le q_j,
  \qquad
  \mathcal C_j
  \Longrightarrow
  c^Y_{j,t}\ge f(\lambda_t)+r_j(q_j)
  \Longrightarrow
  V_{j,t}\le\widetilde V_{j,t}.
\]
\end{proof}

\begin{lemma}[Prequential envelope e-process]
\label{prop:app-preq}
For candidate $j$, fix a bet $\eta_j\in[0,1/q_j)$ and define
\[
  E_{j,t}=E_{j,t-1}\bigl(1+\eta_j(V_{j,t}-q_j)\bigr),
  \qquad E_{j,0}=1.
\]
Let $\widetilde E_{j,t}$ be the same product with $\widetilde V_{j,t}$ in
place of $V_{j,t}$. Then $\widetilde E_{j,t}$ is a nonnegative e-process,
$E_{j,t}\le\widetilde E_{j,t}$ on $\mathcal C_j$, and
\[
  \Pr\!\left(
    \mathcal C_j\cap
    \left\{\sup_{t\ge0}E_{j,t}\ge1/\delta_U\right\}
  \right)
  \le\delta_U .
\]
\end{lemma}

\begin{proof}
\[
  1+\eta_j(\widetilde V_{j,t}-q_j)
  \ge 1-\eta_jq_j>0,
  \qquad
  \E\bigl(1+\eta_j(\widetilde V_{j,t}-q_j)\mid\F_{t-1}\bigr)
  =
  1+\eta_j\bigl(\E(\widetilde V_{j,t}\mid\F_{t-1})-q_j\bigr)
  \le 1 .
\]
Hence $\widetilde E_{j,t}$ is a nonnegative supermartingale, and on
$\mathcal C_j$,
\[
  E_{j,t}
  =
  \prod_{s=1}^t\!\bigl(1+\eta_j(V_{j,s}-q_j)\bigr)
  \le
  \prod_{s=1}^t\!\bigl(1+\eta_j(\widetilde V_{j,s}-q_j)\bigr)
  =
  \widetilde E_{j,t}.
\]
Therefore, by Ville's inequality \citep{ville1939},
\[
  \Pr\!\left(
    \mathcal C_j\cap
    \left\{\sup_{t\ge0}E_{j,t}\ge1/\delta_U\right\}
  \right)
  \le
  \Pr\!\left(\sup_{t\ge0}\widetilde E_{j,t}\ge1/\delta_U\right)
  \le\delta_U .
\]
\end{proof}

\paragraph{Combination.}
Run the one-candidate audit separately for every declared candidate. The
certificate then combines survivors by the flat maximum.

\begin{theorem}[Flat survivor certificate]
\label{thm:app-worst-surviving}
Let $A_t$ be the candidates not deleted by the prequential audit and define
\[
  U_t^{\rm surv}=\max_{j\in A_t}U_{j,t}.
\]
If $A_t=\varnothing$, define $U_t^{\rm surv}=+\infty$, so the rule abstains.
Assume some fixed predeclared anchor $j^\star$ satisfies the coverage and
noise-buffer conditions in Lemma~\ref{lem:app-coverage-audit-null}. If the
selected-configuration lower side has error at most $\alpha_L$, then the rule
\[
  L_t(\hat\lambda_t)\ge U_t^{\rm surv}-\varepsilon
\]
satisfies
\[
  \Pr(\exists t':\hbox{ false }\varepsilon\hbox{-certificate at }t')
  \le \alpha_L+\alpha_U+\delta_U .
\]
\end{theorem}

\begin{proof}
Let
\[
  \mathcal E_L
  =
  \{\exists t':L_{t'}(\hat\lambda_{t'})
  >f(\hat\lambda_{t'})\},
\]
\[
  \mathcal E_U
  =
  \{\exists t',\exists\lambda\in\Lambda:
  u^f_{j^\star,t'}(\lambda)<f(\lambda)\},
  \qquad
  \mathcal E_D
  =
  \mathcal E_U^c\cap\{\exists t':j^\star\notin A_{t'}\}.
\]
The first two events have probabilities at most $\alpha_L$ and $\alpha_U$.
By Lemmas~\ref{lem:app-coverage-audit-null} and~\ref{prop:app-preq},
$\Pr(\mathcal E_D)\le\delta_U$. On the complement of these events, for every
step $t$, the anchor survives and covers the optimum, so
$U_t^{\rm surv}\ge U_{j^\star,t}\ge\fstar$. If the rule fires, then
\[
  f(\hat\lambda_t)\ge L_t(\hat\lambda_t)
  \ge U_t^{\rm surv}-\varepsilon
  \ge \fstar-\varepsilon .
\]
Therefore the false-certificate event is contained in
$\mathcal E_L\cup\mathcal E_U\cup\mathcal E_D$, and
\[
  \Pr(\exists t':\hbox{ false }\varepsilon\hbox{-certificate at }t')
  \le
  \Pr(\mathcal E_L)+\Pr(\mathcal E_U)+\Pr(\mathcal E_D)
  \le
  \alpha_L+\alpha_U+\delta_U .
\]
\end{proof}

We use one fixed predeclared anchor, so the deletion threshold is
$1/\delta_U$ for every candidate. The number of candidates affects computation
and survivor looseness through the maximum over $A_t$, not the risk budget.
The guarantee uses the maximum over survivors; selecting the tightest survivor
requires separate post-selection accounting.

\section{Post-Selection GP Undercoverage}
\label{app:fit-then-certify-risk}

We show here the undercoverage of fitting then certifying.

\begin{proposition}[False-Stop Reduction]
\label{prop:app-risk-undercoverage}
Let $\tau$ be any stopping time at which the rule fires only if
$L_\tau(\hat\lambda_\tau)\ge U_\tau-\varepsilon$. If the lower side satisfies
\[
  \Pr(\exists t':L_{t'}(\hat\lambda_{t'})>f(\hat\lambda_{t'}))\le\alpha_L,
\]
then
\[
  \Pr(\tau<\infty,\ f(\hat\lambda_\tau)<\fstar-\varepsilon)
  \le
  \alpha_L+\Pr(\tau<\infty,\ U_\tau<\fstar).
\]
\end{proposition}

\begin{proof}
Let
\[
  \mathcal E_L
  =
  \{\exists t':L_{t'}(\hat\lambda_{t'})
  >f(\hat\lambda_{t'})\}
\]
and
\[
  \mathcal F
  =
  \{\tau<\infty,\ f(\hat\lambda_\tau)<\fstar-\varepsilon\}.
\]
On $\mathcal F\cap\mathcal E_L^c$,
\[
  U_\tau-\varepsilon
  \le L_\tau(\hat\lambda_\tau)
  \le f(\hat\lambda_\tau)
  < \fstar-\varepsilon,
\]
so $U_\tau<\fstar$. Hence
\[
  \mathcal F
  \subseteq
  \mathcal E_L
  \cup
  \{\tau<\infty,\ U_\tau<\fstar\}.
\]
Therefore
\[
  \Pr(\mathcal F)
  \le
  \Pr(\mathcal E_L)+\Pr(\tau<\infty,\ U_\tau<\fstar)
  \le
  \alpha_L+\Pr(\tau<\infty,\ U_\tau<\fstar).
\]
\end{proof}

Thus the lower side contributes only its declared budget. The remaining term is
controlled by the probability that the claimed global upper bound misses the
true global optimum.

\section{Dimension and Resolution}
\label{app:dimension-resolution}

Global certification is limited by fill distance. The upper certificate must
rule out an $\varepsilon$-better point in the unqueried region, so it becomes
informative only when the queried set resolves the function class.

Under a Lipschitz modulus $\omega(r)\le Lr$, excluding a hidden
$\varepsilon$-improvement requires fill distance $h_t\lesssim\varepsilon/L$.
In a $d$-dimensional box, achieving this requires on the order of
$(L\,{\rm diam}(\Lambda)/\varepsilon)^d$ well-placed evaluations. GP/RKHS
envelopes replace this Lipschitz modulus by a kernel uncertainty term, but they
do not remove the ambient-dimension dependence unless the declared structure is
lower-dimensional, additive, or otherwise stronger.

\begin{theorem}[Finite-History Global Nonidentifiability]
\label{thm:app-imposs}
Fix a finite set of queried configurations
$\Lambda_t=\{\lambda_{t'}:t'\le t\}$. Consider any rule that returns
$\hat\lambda_t\in\Lambda_t$. If $\Lambda$ contains a ball disjoint from
$\Lambda_t$, then without a structural smoothness bound no non-vacuous
finite-history rule can certify global $\varepsilon$-optimality uniformly. More
precisely, for any function $f_0$ on which the rule certifies with positive
probability, there exists a function $f_1$ agreeing with $f_0$ at every queried
configuration but satisfying
$f_1^\star>f_1(\hat\lambda_t)+\varepsilon$.
\end{theorem}

\begin{proof}
Let $B\subset\Lambda$ be a ball with $B\cap\Lambda_t=\emptyset$, and choose
$z\in B$. Let $b:\Lambda\to[0,\infty)$ be continuous, supported in $B$, and
large enough that
\[
  b(z)>f_0(\hat\lambda_t)-f_0(z)+\varepsilon .
\]
Define $f_1=f_0+b$. Then
\[
  f_1(\lambda_{t'})=f_0(\lambda_{t'})
  \qquad\text{for all }t'\le t,
\]
so the two functions induce the same observed history and the same
certification event. But
\[
  f_1^\star
  \ge
  f_1(z)
  =
  f_0(z)+b(z)
  >
  f_1(\hat\lambda_t)+\varepsilon,
\]
since $b(\hat\lambda_t)=0$. Thus any certificate that fires on this history is
false for $f_1$.
\end{proof}

\section{Experiment Protocol}
\label{app:canonical}

All empirical claims use analytic objectives with known global optima. For each
objective and seed, a sequential optimizer produces an ordered sequence
of observations indexed by $t'$. Within each comparison, every method is
evaluated on the same sequence and
may either abstain or certify the current incumbent at each prefix. We record the first
certification time and score it as correct when
$f(\hat\lambda_t)\ge f^\star-\varepsilon$.

We report
\[
  a=\Pr(\mathrm{certify\ and\ correct}),
  \qquad
  b=\Pr(\mathrm{certify\ and\ wrong}),
\]
written as power/risk $a/b$. The compared rows are fixed before aggregation:
the prequential e-value survivor certificate, a random fixed precommitment
drawn from the same candidate set before the run, and the fit-then-certify
diagnostic that tunes a GP envelope on the same history and then treats it as
fixed. Certificate rows use the same selected-configuration lower bound.
The fixed-precommitment comparison uses the $d=3,4$ rows, whereas the
fit-then-certify comparison uses a separate 512-seed $d=3$ RBF length-scale
stress sweep.

The benchmark grid uses $d\in\{2,3,4\}$, RBF, Mat\'ern, and
rational-quadratic kernel-bump objectives, isotropic and anisotropic length
scales, finite-validation observations
$Y(\lambda)\sim\mathrm{Binomial}(n_{\rm val},f(\lambda))/n_{\rm val}$, and
$n_{\rm val}=1000$, $\varepsilon=0.10$ unless stated otherwise. For plots, the
resolution coordinate is $\rho=64^{-1/d}/\ell_{\rm rms}$; it is used only for
binning objectives.

\end{document}